\documentclass[letterpaper, 10 pt]{IEEEconf}
\IEEEoverridecommandlockouts
\usepackage{amsmath,dsfont,amssymb}
\usepackage{amsthm}
\usepackage{graphicx}
\usepackage{amsfonts}
\usepackage{float}
\usepackage{hybridsystem}

\DeclareMathOperator{\Conv}{\mathbf{Conv}}
\DeclareMathOperator{\Poly}{\mathbf{Poly}}
\DeclareMathOperator{\Proj}{\mathbf{Proj}}

\usepackage[colorlinks = true,
            linkcolor = blue,
            urlcolor  = blue,
            citecolor = blue,
            anchorcolor = blue]{hyperref}
\usepackage{algorithmicx}
\usepackage{algorithm}
\usepackage{algpseudocode}
\usepackage{lipsum}
\usepackage{subfig}
\usepackage{amssymb}
\usepackage{bbm}
\usepackage{float}
\usepackage{balance}
\usepackage{color}
\theoremstyle{plain}

\newtheorem{prop}{\textbf{Proposition}}

\theoremstyle{definition}

\theoremstyle{remark}
\newtheorem{rem}{\textbf{Remark}}

\allowdisplaybreaks
\usepackage{algpseudocode}
\algdef{SE}[DOWHILE]{Do}{DoWhile}{\algorithmicdo}[1]{\algorithmicwhile\ #1}%

\usepackage[normalem]{ulem}

\begin{document}

\title{\LARGE \bf
Lidar-based exploration and discretization for mobile robot planning
}
\author{Yuxiao Chen, Andrew Singletary, and Aaron D. Ames
\thanks{The authors are with the Department of Mechanical and Civil Engineering, Caltech,
        Pasadena, CA, 91106, USA. Emails:
        {\tt\small \{chenyx,asinglet,ames\}@caltech.edu}}
}

\maketitle
\begin{abstract}
In robotic applications, the control, and actuation deal with a continuous description of the system and environment, while high-level planning usually works with a discrete description. This paper considers the problem of bridging the low-level control and high-level planning for robotic systems via sensor data. In particular, we propose a discretization algorithm that identifies free polytopes via lidar point cloud data. A transition graph is then constructed where each node corresponds to a free polytope and two nodes are connected with an edge if the two corresponding free polytopes intersect. Furthermore, a distance measure is associated with each edge, which allows for the assessment of quality (or cost) of the transition for high-level planning. For the low-level control, the free polytopes act as a convenient encoding of the environment and allow for the planning of collision-free trajectories that realizes the high-level plan. The results are demonstrated in high-fidelity ROS simulations and experiments with a drone and a Segway.
\end{abstract}
\section{Introduction}\label{sec:intro}
Mobile robotic systems are gaining increasing attention in the past few decades, and their applications can be seen in transportation, exploration of unknown environments, rescue missions post disasters, and surveillance missions. One core functionality is navigation and planning, which has been studied in the literature for more than 40 years. Several important tools include occupancy grid \cite{elfes1989using,murray2000using}, optical flow \cite{dev1997navigation}, potential field \cite{koren1991potential} and roadmap methods \cite{boor1999gaussian,lavalle2001rapidly}. Two closely related problems are mapping and localization since the robot needs to gather information about the environment and localize itself before planning its motion. Existing methods include occupancy grid \cite{elfes1989using,thrun2003learning} and object maps \cite{thrun2002robotic,kostavelis2015semantic}. Simultaneous Localization and Mapping (SLAM) gained enormous popularity by combining localization and mapping \cite{durrant2006simultaneous,cadena2016past}.

While the above-mentioned methods typically deal with a continuous state space (or configuration space), a discrete description of the robot state and task is usually used for high-level planning as it greatly simplifies the problem and serves as an abstraction of the continuous motions \cite{ekvall2008robot}. Moreover, working in a discrete space makes it possible for many powerful planning tools to be applied. One class of tools is the Markov models, including Markov Decision Processes (MDP) and Partially Observed Markov Decision reachability \cite{ong2010planning,zivkovic2006hierarchical,nilsson2018toward}. These models allow for the reasoning of stochastic transitions and typically aim at maximizing the expected reward. Another important class is the temporal logic synthesis tools, which aim at synthesizing controllers that satisfy temporal logic specifications \cite{kress2009temporal,livingston2012backtracking}, e.g., ``the robot should visit an area infinitely often,'' ``once the door opens, the robot should leave the room within 1 minute.''

Most of the existing high-level control synthesis tools assume that a discrete description of the system is given, including the collection of discrete states and the transition relations. However, for mobile robot applications, such a discrete system description is usually not given as the natural operating environment is continuous. One simple approach is to use a grid, such as an occupancy grid, to discretize the space. However, there is a tradeoff between accuracy and scalability. Moreover, the geometry of the grid might not be aligned with the actual free space, causing the low-level control and planning layer difficulty in obtaining the geometry information. There exist tools from the formal methods community that discretizes the state space via reachability analysis \cite{wongpiromsarn2011tulip}, however, this may not be suitable for mobile robot applications since it assumes the environment to be completely known and these methods typically do not scale.
\begin{figure}
    \centering
    \includegraphics[width=1\columnwidth]{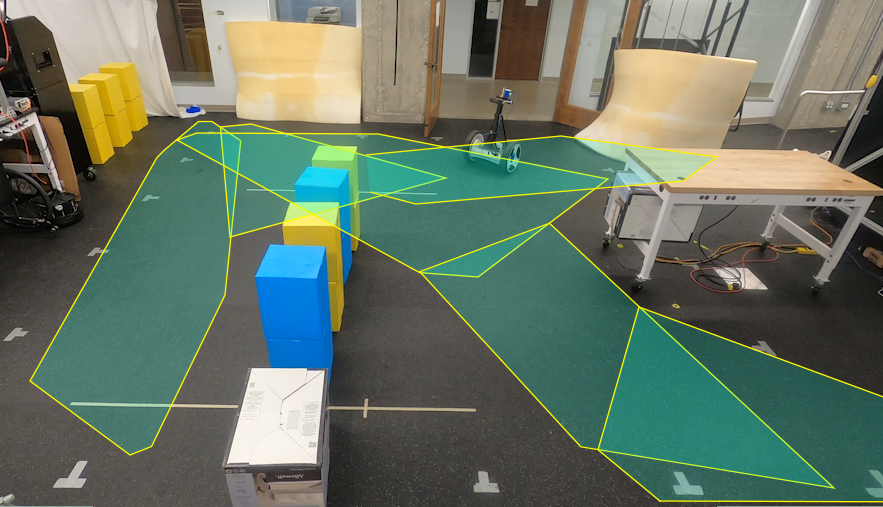}
    \caption{Segway experiment with free polytopes generated from point cloud data}
    \label{fig:intro_segway}
    \vspace{-0.5cm}
\end{figure}

This paper presents a lidar-based exploration and discretization method that explores an unknown environment and generates a discrete representation of the environment in the form of a transition graph for high-level planning. Each discrete state corresponds to a collision-free polytope (referred to as free polytope for the remainder of the paper) and an edge exists between two nodes if the corresponding free polytopes intersect. Once the high-level planning is done on the transition graph, generating a sequence of nodes to visit, the free polytopes then provide the geometry information for the low-level motion planning and control module.

For the remainder of the paper, Section \ref{sec:prelim} present some preliminaries used in the paper, Section \ref{sec:lidar} presents the core functionality, generating free polytopes from lidar point clouds, Section \ref{sec:system} presents the exploration control strategy for the mobile robot, and Section \ref{sec:result} presents the simulation and experiment results.

\section{Preliminaries}\label{sec:prelim}
We first present some preliminaries.

\noindent \textit{\textbf{Nomenclature:}} $\mathbb{R}$ and $\mathcal{R}^n$ denote the real numbers and the $n$-dimensional Euclidean space, respectively. $||x||$ denote the Euclidean distance of a vector $x\in\mathbb{R}^n$. $\Poly(P,q)=\left\{x\mid Px\le q\right\}$ denotes a polytope defined with matrix $P,q$. For a polytope $\mathcal{P}$, $\mathcal{P}.V$ denotes the set of its vertices, $\mathcal{P}.H$ denotes the set of its separating hyperplanes. $\Conv(S)$ denotes the convex hull of a set $S$,  $\Proj_{S}(x)\doteq \mathop{\arg\min}\nolimits_{z\in S}||x-z||$ denotes the projection of the point $x$ onto a set $S$. If $S$ is a polytope, the projection can be obtained with quadratic programming; if $S$ is a union of polytopes, the projection is obtained by projecting $x$ onto every polytope in the union and take the point with minimum distance. For a discrete set $S$, $|S|$ denotes its cardinality. For a compact set $S$, $int(S)$ is its interior, and $\partial S$ denotes its boundary.

\newsec{Transition graphs.} The outcome of the proposed discretization algorithm is a transition graph, which consists of a collection of nodes $V$, a collection of edges $E$, and a distance function $E\to\mathbb{R}$. The nodes in the transition graph are the free polytopes generated from the lidar point cloud, and an edge exists between two nodes if the two free polytopes intersect. Two nodes are connected if there exists a path that connects them. The path of the minimum total distance can be efficiently solved with the Dijkstra algorithm or the $A^*$ algorithm. Later we shall show how we maintain the transition graph while exploring an unknown environment.
\begin{rem}
  The transition graph can be modified to a Markov decision process if the transition probability can be estimated. This may be desirable in cases such as a rover traversing rough terrains.
\end{rem}

\newsec{IRIS.}
The core of the discretization process is to construct free polytopes, i.e., polytopic sets in the state space that are collision-free, from lidar point clouds. We adopt the Iterative Regional Inflation by Semidefinite programming (IRIS) \cite{deits2015computing} algorithm as the main computation engine. There exist other algorithms such as fitting free balls in \cite{gao2016online}, but we found polytopes particularly appealing due to the convenience of encoding the free space in low-level planning, and the free space constructed from IRIS tend to be larger than other algorithms. The IRIS algorithm works by iteratively alternating between two steps: (1) a quadratic program that generates a set of hyperplanes to separate a convex region of space from the set of obstacles and (2) a semidefinite program that finds a maximum-volume ellipsoid inside the polytope intersection of the obstacle-free half-spaces defined by those hyperplanes. The input to the algorithm is a collection of polytopic obstacles and a starting location, and the output is a collision-free polytope. See \cite{deits2015computing} for more details.

\section{Free polytopes from lidar}\label{sec:lidar}

This section presents the core functionality, generating free polytopes from lidar point cloud. There are three main modules of the for this process: (1) preprocessing of the point cloud into polytopic obstacles (2) constructing free polytope with IRIS (3) post-processing of the free polytope.

\newsec{Preprocessing the point cloud.} To begin with, the raw point cloud data need to be filtered to remove the noisy points. In addition, the point cloud needs to be grouped into polytopic obstacles, which is used in the second step by the IRIS algorithm. We combine the two tasks into one algorithm, summarized in Algorithm \ref{alg:preprocessing}. It takes the lidar's position $x_r$ and the point cloud $C$ as inputs, and outputs a collection of polytopic obstacles, where the following subroutines are used:
\begin{itemize}
\item \textsc{Crop\_point\_cloud}($x_r,C,\overline{d}$) crops the point cloud points so that the maximum distance to the robot position $x_r$ is $\overline{d}$.
\item \textsc{Random\_select}($C$) randomly selects a seed point from the point cloud $C$.
\item \textsc{Get\_neighbors}($C,x_s$) returns all neighboring points in $C$ within a ball around $x_s$.
\item \textsc{Regression}($C_s$) performs linear regression to fit a hyperplane to $C_s$ as
\begin{equation*}
\mathop{\min}\limits_{k,c,k_1=1} \sum_{x\in C_s} ||k^\intercal x+c||^2
\end{equation*}
\item \textsc{Grow\_plane}($C,x_s,k,c,\epsilon_1,\epsilon_2||x_r-x_s||$) grows a set $C_p$ consisting of points satisfying $||k^\intercal x+c||<\epsilon_1$ and the minimum distance between $x$ and other points within $C_p$ is smaller than $\epsilon_2$ times $||x_r-x_s||$, the distance between the robot position and the seed point. We make the threshold proportional to the distance from the point cloud to the robot because the distance between points in the point cloud grows linearly with distance to the robot. $C_p$ is initialized as $\{x_s\}$ with only one element, the seed point. Then it iteratively add points in $C$ that satisfies the two conditions until no point can be added.
\end{itemize}

During each iteration, if $|C_p|$ is larger than the threshold $n_3$, its convex hull is added to the obstacle set and $N_{fail}$ drops to zero; otherwise $N_{fail}$ increase by 1. If the procedure fails $n_2$ times in a row, the process is terminated.
\begin{prop}
Algorithm \ref{alg:preprocessing} terminates in finite time
\end{prop}
\begin{proof}
For every $n_2$ iterations, at least $n_3$ points are removed from $C$, therefore the while loop terminates in finite interations.
\end{proof}

\begin{algorithm}
    \caption{Preprocessing of point cloud}
    \label{alg:preprocessing}
    \begin{algorithmic}[1] % The number tells where the line numbering should start
        \Procedure{Preprocessing}{$x_r,C,\overline{d},n_1,n_2,n_3,\epsilon_1,\epsilon_2$}
        	\State \textsc{Crop\_point\_cloud}($x_r,C,\overline{d}$)
           	\State $\mathcal{O}\gets\emptyset$,$N_{fail}\gets 0$
           	\While{$|C|\ge n_1$ \textbf{and} $N_{fail}<n_2$}
           	\State$C_{s}\gets$ \textsc{Random\_select}($C$)
           	\State $C_{s}\gets$ \textsc{Get\_neighbors}($C,x_s$)
           	\State $k,c\gets$ \textsc{Regression}($C_s$)
           	\State $C_p\gets$\textsc{Grow\_plane}($C,x_s,k,c,\epsilon_1,\epsilon_2||x_r-x_s||$)
           	\If{$|C_p|>n_3$}
           		\State Add $\Conv (C_p)$ to $\mathcal{O}$, $C\gets C\backslash C_p$
           		\State $N_{fail}\gets 0$
           	\Else
           		\State $N_{fail}\gets N_{fail}+1$
           	\EndIf
           	\EndWhile
           	\State \Return $\mathcal{O}$
        \EndProcedure
    \end{algorithmic}
\end{algorithm}

Algorithm \ref{alg:preprocessing} turns the point cloud data into a collection of polytopic obstacles. In the meantime, noisy points are removed as they are often separated from the rest of the points and thus cannot be merged into any of the polytopic obstacles. We then need to construct an obstacle-free set given the obstacles, which is done by IRIS. The bound for IRIS is chosen as a box that contains all the points in $C$ so that the polytope obtained by IRIS does not ``slip out'' of the point cloud.

\newsec{Postprocessing the free polytopes.} Once the free polytope $\mathcal{P}$ is generated by IRIS, it is postprocessed before added to the graph of connected free polytopes $\mathcal{G}$. First, the polytope is shrunk as follows:
\begin{equation}
\mathcal{P} = \mathcal{P} \ominus \mathcal{B}(0,r),
\end{equation}
where $\ominus$ is the Minkowski difference and $\mathcal{B}(0,r)$ is a ball at the origin with radius $r$, which is the radius of the robot collision circle. This makes $\mathcal{P}$ at least $r$ distance from any obstacle, thus the robot only need to keep its coordinate within $\mathcal{P}$. For polytopes, the shrinking operation is simple. Assume that $\mathcal{P}=\Poly(P,q)$, we have
\begin{equation*}
\begin{aligned}
	\mathcal{P}\ominus \mathcal{B}(0,r)&=\Poly(P,q'),\\
	q_i'&=q_i-r ||P_i||,
	\end{aligned}
\end{equation*}
where $P_i$ is the $i$th row of $P$, $q_i$ is the $i$th entry of $q$.

Before adding $\mathcal{P}$ to the union of free polytopes $\mathcal{FS}=\bigcup\limits_i \mathcal{P}_i$, we would like to reduce the size of intersections between $\mathcal{P}$ and existing free polytopes in $\mathcal{FS}$. This is because each free polytope represents a discrete state in the high-level planning module, and it is preferred that the overlap between the free polytopes is minimized. To do this, a polytope slicing algorithm is developed.

The idea is to add additional separating hyperplanes to the new polytope $\mathcal{P}$ so that (1) no point that is not contained in the existing free polytopes in $\mathcal{FS}$ is removed (2) the overlap between $\mathcal{P}$ and $\mathcal{FS}$ is minimized. Fig. \ref{fig:poly_int} demonstrates one example with $\mathcal{P}_1$ and $\mathcal{P}_2$ intersecting, and $\mathcal{H}$ is the desired hyperplane.
\begin{figure}
\centering
  \includegraphics[width=0.9\columnwidth]{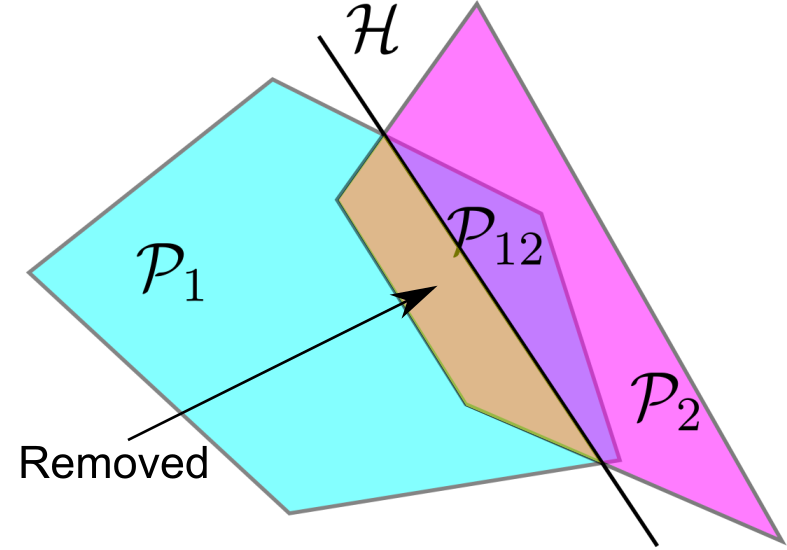}
  \caption{Adding an additional hyperplane $\mathcal{H}$ to $\mathcal{P}_2$ to reduce the overlap between $\mathcal{P}_1$ and $\mathcal{P}_2$.}\label{fig:poly_int}
  \vspace{-0.5cm}
\end{figure}
First observe that in order to keep any point $x\in\mathcal{P}_2,x\notin\mathcal{P}_1$, if we only add one new polytope, the smallest polytope is $\Conv(\mathcal{P}_1\backslash\mathcal{P}_1)$. In fact, $\Conv(\mathcal{P}_1\backslash\mathcal{P}_1)$ turns out to be the set we seek in the simple case where only one new polytope is added. To illustrate the idea, consider the situation depicted in Fig. \ref{fig:poly_int} with two polytopes $\mathcal{P}_1$ and $\mathcal{P}_2$, $\mathcal{P}_2$ being the new polytope to be added. Let $\mathcal{P}_{1,2}\doteq\mathcal{P}_1\cap\mathcal{P}_2$, which is a polytope itself. Its vertices can be categorized into three types: (1) vertices of $\mathcal{P}_1$, characterized as $V_1\doteq \{v\in \mathcal{P}_{1,2}.V|v\in int(\mathcal{P}_2)\}$ (2) vertices of $\mathcal{P}_2$, denoted as $V_2\doteq \{v\in \mathcal{P}_{1,2}.V|v\in int(\mathcal{P}_1)\}$ (3) new vertices generated by the intersection, $V_3\doteq\{v\in \mathcal{P}_{1,2}.V|v\in \partial \mathcal{P}_1\cap\partial\mathcal{P}_2\}$.
\begin{prop}
$\Conv(\mathcal{P}_1\backslash\mathcal{P}_2)$ can be computed as $\Conv(\mathcal{P}_1\backslash\mathcal{P}_2)=\Conv((\mathcal{P}_1.V\backslash V_1) \cup V_3 )$.
\end{prop}
\begin{proof}
First we show $\Conv((\mathcal{P}_1.V\backslash V_1) \cup V_3 )\subseteq \Conv(\mathcal{P}_1\backslash\mathcal{P}_2)$. This is straightforward since the set $\mathcal{P}_1.V\backslash V_1) \cup V_3\subseteq \mathcal{P}_1\backslash\mathcal{P}_2$. To show the other direction, note that convex hull preserves vertices, i.e., the vertices of $\Conv(\mathcal{P}_1\backslash\mathcal{P}_2)$ is a subset of $\mathcal{P}_1.V \cup V_3$. For any $v\in V_1$, since it is a vertice of $\mathcal{P}_1$, one can find a separating hyperplane between $v$ and $\mathcal{P}_1.V\backslash v\cup V_3$. Therefore, $v\notin \Conv(\mathcal{P}_1\backslash \mathcal{P}_2)$. Similarly, all of $V_2$ is not inside $\Conv(\mathcal{P}_2\backslash \mathcal{P}_1)$.
%Moreover, it can be shown that $(\mathcal{P_1}\cup\mathcal{P}_2)\subseteq(\Conv(\mathcal{P}_1\backslash \mathcal{P}_2) \cup \Conv(\mathcal{P}_2\backslash \mathcal{P}_1))$, therefore, $V_1\subseteq \Conv(\mathcal{P}_2\backslash \mathcal{P}_1)$, and $V_2\subseteq \Conv(\mathcal{P}_1\backslash \mathcal{P}_2)$.
Now consider the potential vertices of $\Conv(\mathcal{P}_1\backslash \mathcal{P}_2)$, we know that $V_1$ is not inside $\Conv(\mathcal{P}_1\backslash \mathcal{P}_2)$, so what is left is $(\mathcal{P}_1.V\backslash V_1) \cup V_3$, and since convex hull preserves vertices,  $\Conv((\mathcal{P}_1\backslash\mathcal{P}_2)\subseteq\Conv(\mathcal{P}_1.V\backslash V_1) \cup V_3)$.
\end{proof}

In practice, the situation can be more complicated. For example, the new free polytope can intersect with multiple existing free polytopes, $\mathcal{P}_1\backslash\mathcal{P}_2$ can be disconnected, in which case further shrinking is possible by splitting $\mathcal{P}_1$ and add multiple new polytopes. However, note that any new separating hyperplane $\mathcal{H}$ that is added to shrink $\mathcal{P}_1$ is always purely determined by some vertices within $V_3$, therefore, we can enumerate all hyperplanes determined by vertices in $V_3$ to shrink and split $\mathcal{P}_1$ into polytopes that preserves all the newly found free space and minimizes overlap. We omit the details here.

\begin{algorithm}
    \caption{Adding new free space}
    \label{alg:adding}
    \begin{algorithmic}[1] % The number tells where the line numbering should start
        \Procedure{Add\_new\_poly}{$\mathcal{P},\mathcal{FS}$}
        	\State $\mathcal{H}s\gets \emptyset$
        	\For{$\mathcal{P}'\in\mathcal{FS}$}
        		\State Calculate $(\mathcal{P}\cap\mathcal{P}').V$, Identify $V_1,V_2,V_3$
        		\State $\mathcal{H}s\gets \mathcal{H}s\cup\{\Conv(V_3).H\}$
        	\EndFor
           	\For{$\mathcal{H}\in\mathcal{H}s$}
           		\If{\textsc{Add\_criteria}($\mathcal{FS},\mathcal{P}\cap\mathcal{H})$}
           			\State $\mathcal{FS}\gets\mathcal{FS}\cup\{\mathcal{P}\cap\mathcal{H}\}$
           		\EndIf
           	\EndFor
        \EndProcedure
    \end{algorithmic}
\end{algorithm}

The process of shrinking a new free polytope and adding to the existing union of free polytopes is summarized in Algorithm \ref{alg:adding}, where $\mathcal{P}\cap\mathcal{H}$ denotes the intersection of $\mathcal{P}$ and the halfspace generated by $\mathcal{H}$, and \textsc{Add\_criteria} is a procedure that determines whether a new polytope should be added to $\mathcal{FS}$ based on whether the new free polytope leads to a sufficient growth of the total free space.

\section{System structure}\label{sec:system}
This section presents the system design for the mobile robot that carries the lidar to perform the discretization task. We do not specify a particular type of mobile robot, but rather treat the robot as a point mass with a radius $r$ that can be ordered to move around with the help of the low-level planning and control layer. In the simulation and experiment section, we use an Unmanned Aerial Vehicle (UAV) and a Segway as the mobile robot.

There are three key components for the mapping and discretization task: (1) high-level planning module (2) low-level planning and control module (3) mapping module

\newsec{High-level planning module.}
The high-level planning module plans the desired waypoint for the mobile robot to visit next, which does not consider the robot dynamics and there is no safety guarantee. We adopt the planner design in \cite{singletary2020safety} which works in tandem with the Octomap library \cite{hornung2013octomap} and utilizes the point cloud data to identify the visited space and the frontier of exploration. To ensure safety, the desired waypoint given by the planner is projected to $\mathcal{FS}$ as the actual waypoint. At the beginning of the operation, it is assumed that the mobile robot starts at an initial condition with sufficient clearance around it. The mobile robot will first identify the first free polytope at the initial position and add it to $\mathcal{FS}$ before it starts moving towards the waypoint given by the high-level planning module.

\newsec{Low-level planning and control module.}
The low-level planning and control module receives the waypoint from the high-level planning module and plans a path towards the waypoint based on the dynamics of the mobile robot. In the Segway example, we tested a Model Predictive Control (MPC) controller and a Linear Quadratic Regulator (LQR)-based controller to plan the trajectory to the waypoint and calculate the torque input for the Segway. In the UAV example, a PD controller is used to track the desired velocity pointing at the waypoint, and a Control Barrier Function (CBF) supervisory controller directly using lidar point clouds runs on top of the PD controller to make sure that the UAV is not colliding with any obstacles, see \cite{singletary2020safety} for detail.

\newsec{Mapping module.}
The mapping module mainly uses the lidar-based free polytope generation described in Section \ref{sec:lidar} with some auxiliary functions.

First, a transition graph $\mathcal{G}$ is maintained with nodes being the free polytopes and an edge exists between two free polytopes if they intersect. For each free space $\mathcal{P}$, its centroid $C_\mathcal{P}$ is also stored with $\mathcal{P}$ in $\mathcal{G}$. Each edge is associated with two objects, the centroid of the intersection of the two free polytopes, and the transition cost. The estimation of the transition cost differs in different applications. We simply take it to be the distance of the path connecting the two centroids that pass through the centroid of the intersection:
\begin{equation*}
d(\mathcal{P}_1,\mathcal{P}_2)=||C_{\mathcal{P}_1}-C_{\mathcal{P}_1\cap\mathcal{P}_2}||+||C_{\mathcal{P}_2}-C_{\mathcal{P}_1\cap\mathcal{P}_2}||.
\end{equation*}
For ground robots such as the Segway, it can depend on information about the terrain and the difficulty of passing. The benefit of storing the centroid of $\mathcal{P}_1\cap \mathcal{P}_2$ is that the straight path from the centroid of $\mathcal{P}_1$ to $C_{\mathcal{P}_1\cap\mathcal{P}_2}$ is guaranteed to lie within $\mathcal{P}_1$, the same is true for $\mathcal{P}_2$. Therefore, as long as the mobile robot can follow straight paths, a collision-free transition from $\mathcal{P}_1$ to $\mathcal{P}_2$ can be achieved.
\begin{figure}
  \centering
  \includegraphics[width=1\columnwidth]{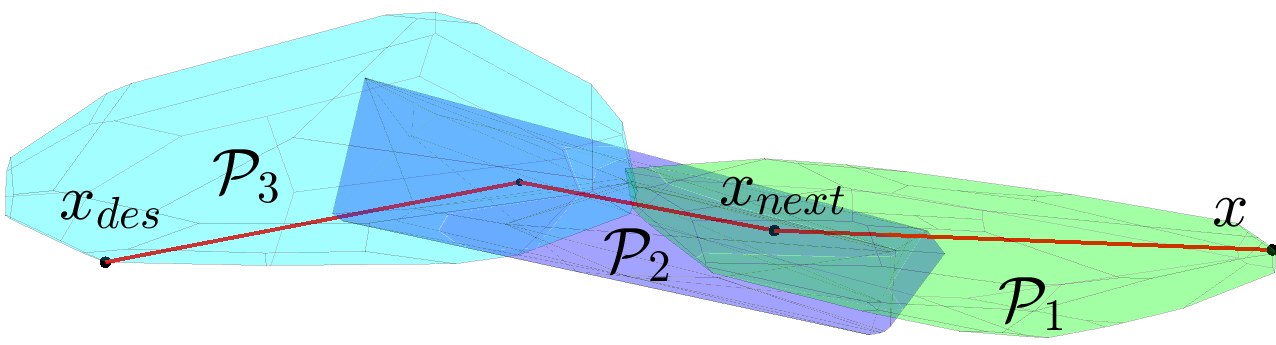}
  \caption{Transition between free polytopes.}\label{fig:nav_poly}
  \vspace{-0.5cm}
\end{figure}

The overall exploration strategy for the robot is summarized in Algorithm \ref{alg:strategy}. $x$ is the current position of the robot, $\mathcal{FS}$ is the free space consisting of multiple free polytopes, and $\mathcal{G}$ is the transition graph. The robot operates in two modes, \textsc{scan} and \textsc{move}. In \textsc{scan} mode, the robot would gather the point cloud and execute subroutines $\textsc{Update\_free\_space}$ and $\textsc{Update\_discrete\_graph}$ to update the free space $\mathcal{FS}$ and the transition graph $\mathcal{G}$. In \textsc{move} mode, the robot would navigate itself to $x_{des}$ while staying inside the free space $\mathcal{FS}$. Since the current position of the robot $x$ and the waypoint $x_{des}$ may not share the same free polytope, we use \textsc{Generate\_x\_next} to generate an intermediate waypoint to navigate the robot to $x_{des}$, which is done with the following 3 steps:
\begin{itemize}
  \item identify free polytopes $\mathcal{P}_1$ and $\mathcal{P}_2$ that contains $x$ and $x_{des}$, respectively,
  \item do Dijkstra search to obtain a discrete path from $\mathcal{P}_1$ to $\mathcal{P}_2$,
  \item if $x\in \mathcal{P}_2$, return $x_{next}=x_{des}$, else, return a point inside the intersection of $\mathcal{P}_1$ and the next free polytope on the discrete path.
\end{itemize}
Fig. \ref{fig:nav_poly} shows an example run of the \textsc{Generate\_x\_next} subroutine where the intermediate waypoint is inside the intersection of $\mathcal{P}_1$ and $\mathcal{P}_2$, which ensures that the straight path from $x$ to $x_{next}$ is within $\mathcal{P}_1$.

This subroutine together with the $\textsc{Navitation}$ ensures that $x$ can reach $x_{des}$ while staying inside $\mathcal{FS}$. When the planner sends a new waypoint $x_{des}$, it is projected onto $\mathcal{FS}$ as the new waypoint for the robot.
\begin{algorithm}
    \caption{Autonomous exploration and discretization}
    \label{alg:strategy}
    \begin{algorithmic}[1]
        \State $\text{mode}\gets \textsc{scan}$, $\mathcal{FS}\gets\emptyset$, $\mathcal{G}\gets (\emptyset,\emptyset)$, $x_{des}\gets x$,
        \While{$\neg$ TERMINATE}
            \If {$\text{mode}==\textsc{scan}$}
                \State $\mathcal{FS}\gets\textsc{Update\_free\_space}(x,\mathcal{FS})$
                \State $\mathcal{G}\gets\textsc{Update\_discrete\_graph}(\mathcal{FS})$
                \State $\text{mode}\gets\textsc{move}$
            \ElsIf{$\text{mode}==\textsc{move}$}
                \State $x_{next}\gets \textsc{Generate\_x\_next}(x,x_{des},\mathcal{G},\mathcal{FS})$
                \State $\textsc{Navigate}(x_{next})$
                \If{$||x-x_{des}||\le\epsilon$}
                    \State $\text{mode}\gets\textsc{scan}$
                \EndIf
            \EndIf
        	\If {New\_waypoint}
                \State $x_{des}\gets\textsc{Receive\_waypoint()}$
                \State $x_{des}\gets\Proj_{\mathcal{FS}}(x_{des})$
            \EndIf
        \EndWhile
    \end{algorithmic}
\end{algorithm}

\begin{rem}
  The free polytopes can also be used by other navigation algorithms. One particular example is the model predictive control, in which the state predictions are constrained inside the union of the free polytopes:
  \begin{equation*}
    \forall t, x(t)\in\bigcup_i \mathcal{P}_i,
  \end{equation*}
where $\bigcup_i \mathcal{P}_i$ is the collection of free polytopes, and each $\mathcal{P}_i=\{x|A_i x\le b_i \}$ is defined with linear constraints. The above constraint can be conveniently encoded as a mixed-integer linear constraint, and be efficiently solved.
\end{rem}

\section{Simulation and experiment results}\label{sec:result}
We demonstrate the proposed algorithm with two examples. The first example is a ROS simulation of a drone carrying a Velodyne lidar exploring a mining cave. The second example is the experiment of a Segway robot carrying a D435 Realsense camera that maps the lab environment.

\subsection{Drone simulation}
The ROS simulation uses a 17-dimensional quadrotor model. The state vector $x = [\mathbf{r}, \mathbf{v}, \mathbf{q}, \mathbf{w}, \mathbf{\Omega}]^\intercal$ where $\mathbf{r}$ is the position $[x,y,z]^\intercal$ in $\mathbb{R}^3$, $\mathbf{v}$ is the velocity $[v_x,v_y,v_z]^\intercal$ in the world frame, $\mathbf{q}$ is the quaternion $[q_w,q_x,q_y,q_z]^\intercal$, $\mathbf{w}$ is the angular velocity vector $[w_x,w_y,w_z]^\intercal$ in the body frame, and $\mathbf{\Omega}$ is the vector of angular velocities of the propellers, $[\Omega_1,\Omega_2,\Omega_3,\Omega_4]^\intercal$. The control input is the voltages applied at the motors $u = [V_1,V_2,V_3,V_4]^\intercal$. The drone is equipped with a Velodyne lidar, which is simulated with the velodyne simulator ROS package \cite{velodyne}. A snapshot of the ROS simulation is shown in Fig. \ref{fig:drone_sim}, where the drone follows the path planned by the planner and the point cloud are demonstrated with colors. The point cloud of the Velodyne simulator is omnidirectional, and is fed to the preprocessing algorithm in Algorithm \ref{alg:preprocessing} to generate the set of obstacles. IRIS then calculates a free space from the set of obstacles.
\begin{figure}
  \centering
  \includegraphics[width=0.8\columnwidth]{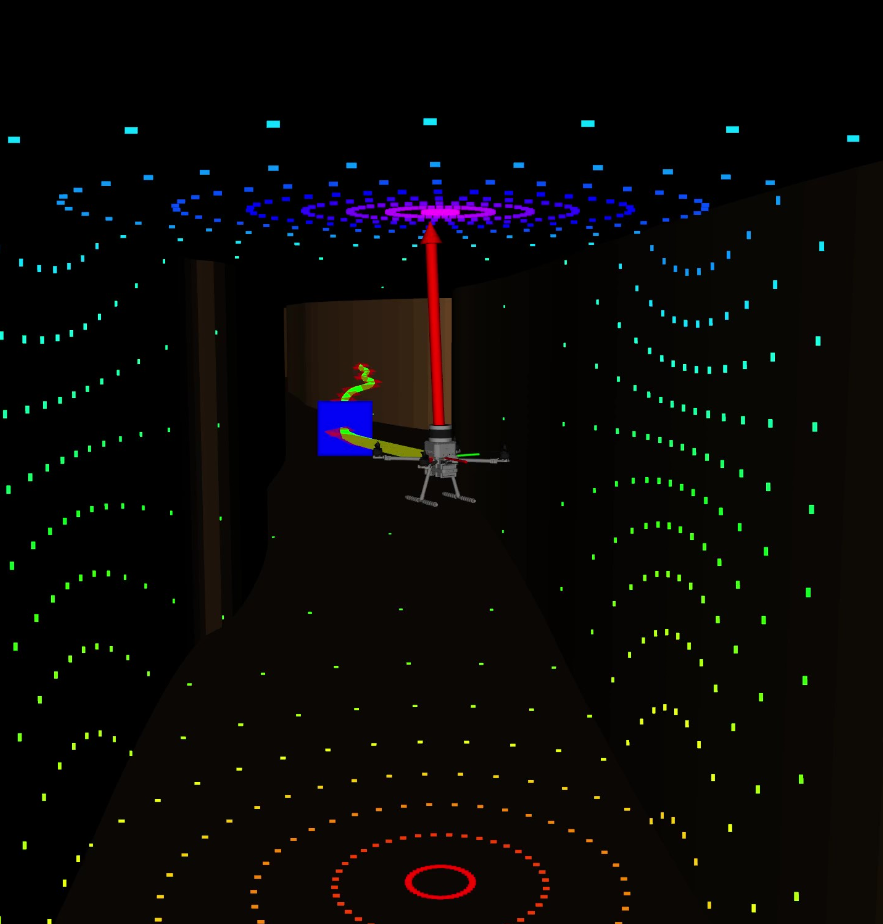}
  \caption{Snapchot of the ROS simulation.}\label{fig:drone_sim}
\end{figure}
Fig. \ref{fig:uav_fs} shows the free space generated by the drone, which contains 136 free polytopes after 18 minutes of run time. The blue curve shows the trajectory of the drone while mapping the mining cave environment. The 136 polytopes form a connected free space $\mathcal{FS}$ where the drone can reach any point in $\mathcal{FS}$ from any other points in $\mathcal{FS}$.
\begin{figure}
  \centering
  \includegraphics[width=1\columnwidth]{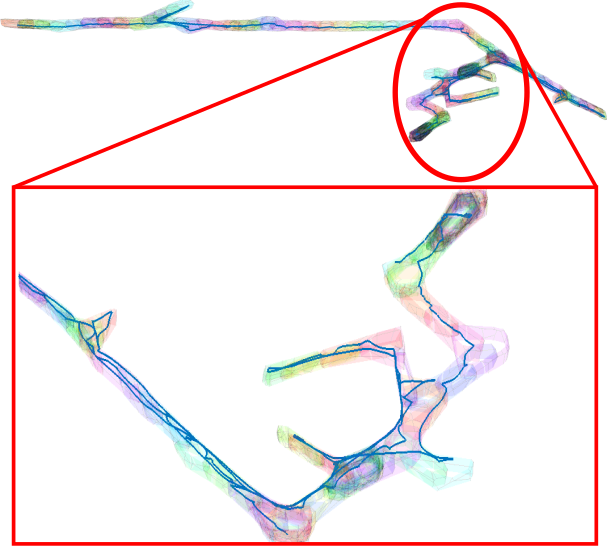}
  \caption{Free polytopes generated in the UAV simulation.}\label{fig:uav_fs}
  \vspace{-0.5cm}
\end{figure}

\subsection{Segway experiment}
The algorithm is also tested with experiments on a Segway platform. The Segway is custom made with 7 states: $x=[X,Y,v,\theta,\dot{\theta},\psi,\dot{\psi}]$, where $X$, $Y$ are the Cartesian coordinates, $v$ is the forward velocity, $\theta$ and $\psi$ are the yaw and pitch angle. The input is the torque of the two motors. The Segway is equipped with a D435 Realsense depth camera with a $86^{\circ}\times57^{\circ}$ field of view
and $10$m range. The D435 camera generates a colored point cloud, yet we treat it as an uncolored point cloud. Since the point cloud is not omnidirectional, the Segway would turn $360^{\circ}$ to obtain point cloud in every direction before updating $\mathcal{FS}$.

We use an extended Kalman filter to obtain the pose estimation of the Segway, which relies on an IMU, two wheel encoders, and a T265 Realsense tracking camera for sensor input. The \textsc{Navigation} module orders the Segway to follow straight lines between its current position and the intermediate waypoint $x_{next}$ by sending velocity and yaw rate command, which is then tracked with a low-level controller designed with LQR.
\begin{figure}
  \centering
  \includegraphics[width=1\columnwidth]{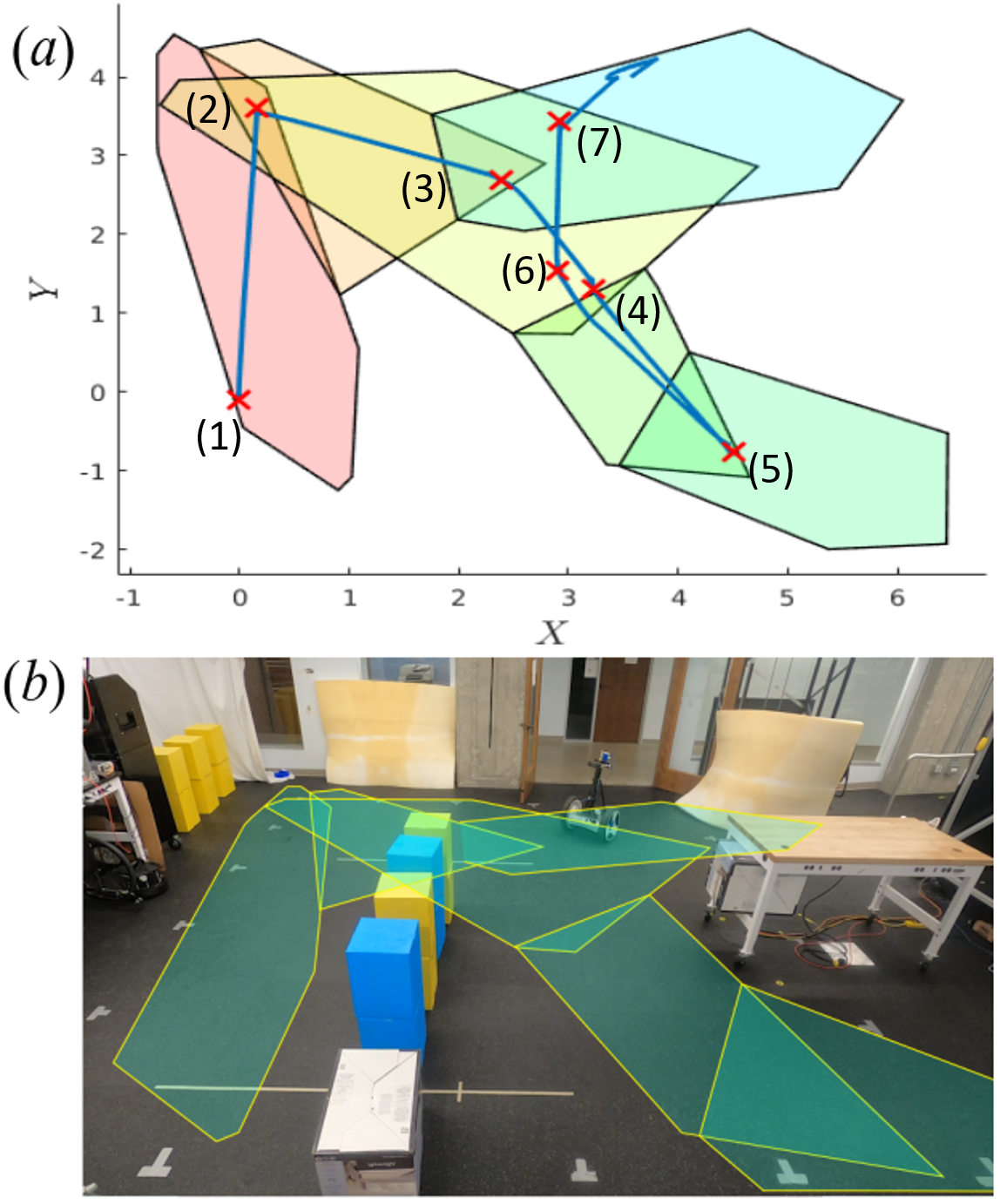}
  \caption{Segway exploration of AMBER lab.}\label{fig:segway_res}
  \vspace{-0.5cm}
\end{figure}

\begin{figure}
  \centering
  \includegraphics[width=1\columnwidth]{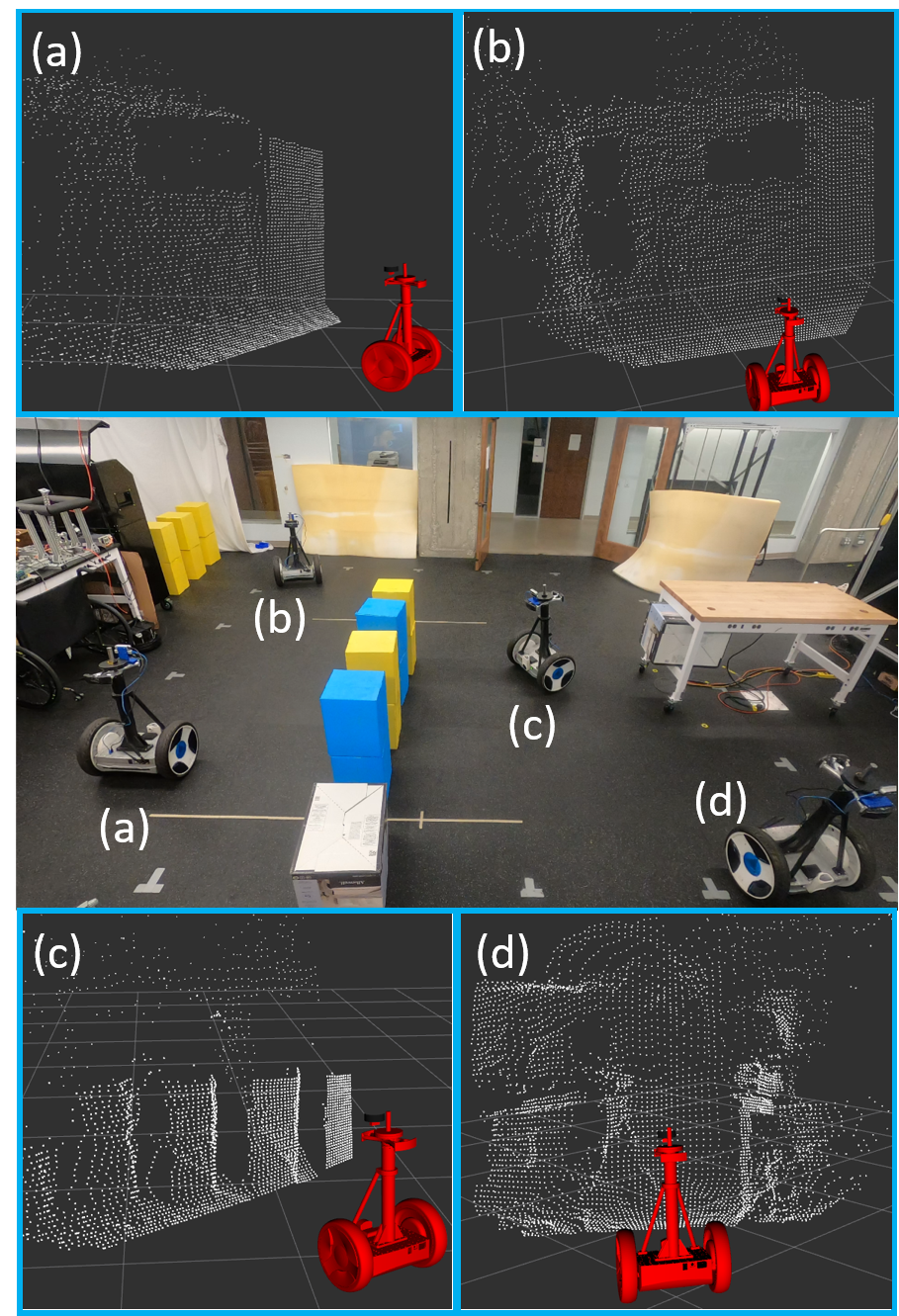}
  \caption{Point cloud data from D435 depth camera during experiment}\label{fig:exp_pc}
\end{figure}

Fig. \ref{fig:segway_res} shows the Segway experiment inside the AMBER lab with obstacles such as tables, boxes, and walls. Fig. \ref{fig:segway_res}(a) shows the free polytopes, the trajectory of the Segway in blue, and the spots of scanning in red crosses with the numbers showing the order of the scans. Fig. \ref{fig:segway_res}(b) shows the free polytopes in the actual environment. The Segway performed 7 scans, resulting in 6 free polytopes (one of the scans did not find free polytopes worthy of adding to $\mathcal{FS}$), and the Segway stayed inside $\mathcal{FS}$ throughout the experiment. A video of the experiment can be found in \href{https://vimeo.com/474287456}{LINK}. It is noticed that one of the free polytopes intersects with the box obstacles, which is probably due to the preprocessing module removing points in the point cloud that are actually obstacles.

Fig. \ref{fig:exp_pc} shows the raw point cloud from the depth camera during the experiment with the Segway position in the middle figure. Snapshot (a) was taken at the beginning of the experiment, (b) was taken before the second scan, (c) was taken during the third scan when the Segway was facing the boxes, and (d) was taken during the fifth scan. Snapshot (b) showed that the laser beams passed through the glass windows and missed the glass surface, which is the reason why the algorithm did not work well when we tested in a hallway surrounded by giant glass windows.

\section{Conclusion and discussion}
We propose a lidar-based exploration and discretization algorithm that generates a collection of free polytopes for mobile robot navigation. The goal is to bridge the gap between high-level planning, which deals with a discrete representation of the environment, and low-level planning, which deals with a continuous representation of the environment. On the high-level, the discretization algorithm generates a transition graph, which can be used by path planning algorithms such as $A^*$ and temporal logic planning tools such as Tulip. With an estimation of transition probability, high-level planning can also be solved as a Markov decision process. On the low-level, the free polytopes are a convenient encoding of the environment geometry, which can be used to plan collision-free trajectories via tools such as Model Predictive Control.

However, we acknowledge that the proposed algorithm needs improvement under the presence of moving obstacles, and existing problems with lidar sensors still exist, such as handling special materials like glass.
\newpage
\balance
\renewcommand{\baselinestretch}{0.94}
\bibliographystyle{myieeetran}
\bibliography{my_bib}
\end{document}